\documentclass{ecai}
\usepackage{times}
\usepackage{graphicx}
\usepackage{latexsym}
\usepackage{amsmath}
\usepackage{amssymb}

\providecommand{\doarxiv}{true}
\usepackage{ifthen}
\newboolean{isarxiv}
\setboolean{isarxiv}{\doarxiv} 
\ifthenelse{\boolean{isarxiv}}{%
\newcommand{\arxiv}[1]{#1}
\newcommand{\notarxiv}[1]{}
}{
\newcommand{\arxiv}[1]{}
\newcommand{\notarxiv}[1]{#1}
}
\newcommand{\arxivalt}[2]{\ifthenelse{\boolean{isarxiv}}{#1}{#2}}
\newcommand{\arxivaltr}[2]{\ifthenelse{\boolean{isarxiv}}{#2}{#1}}
\newcommand{\narxiv}[1]{\notarxiv{#1}}

\newcommand{\truncated}{curve-normalized}

\newcommand{\lovasz}{Lov\'asz}

\newtheorem{lemma}[theorem]{Lemma}

\newtheorem{corollary}[theorem]{Corollary}

\newcommand{\curvf}[1]{\ensuremath{\kappa_{#1}}}
\newenvironment{proof}{\paragraph{Proof:}}{\hfill$\square$}

\begin{document}

\title{Robust Submodular Minimization with Applications to Cooperative Modeling}

\author{Rishabh Iyer \institute{University of Texas at Dallas,
USA, email: rishabh.iyer@utdallas.edu} }

\maketitle
\bibliographystyle{ecai}

\begin{abstract}
Robust Optimization is becoming increasingly important in machine learning applications. This paper studies the problem of robust submodular minimization subject to combinatorial constraints. Constrained Submodular Minimization arises in several applications such as co-operative cuts in image segmentation, co-operative matchings in image correspondence etc. Many of these models are defined over clusterings of data points (for example pixels in images), and it is important for these models to be robust to perturbations and uncertainty in the data. While several existing papers have studied robust submodular maximization, ours is the first work to study the minimization version under a broad range of combinatorial constraints including cardinality, knapsack, matroid as well as graph based constraints such as cuts, paths, matchings and trees. In each case, we provide scalable approximation algorithms and also study hardness bounds. Finally, we empirically demonstrate the utility of our algorithms on synthetic and real world datasets.\looseness-1
\end{abstract}

%
\section{Introduction}
\label{introduction}
Submodular functions provide a rich class of expressible models for a
variety of machine learning problems. They occur
naturally in two flavors. In minimization problems, they model
notions of cooperation, attractive potentials, and economies of scale,
while in maximization problems, they model aspects of coverage,
diversity, and information. A set function $f: 2^V \to \mathbb R$ over
a finite set $V = \{1, 2, \ldots, n\}$ is \emph{submodular} 
\cite{fujishige2005submodular} if for all
subsets $S, T \subseteq V$, it holds that $f(S) + f(T) \geq f(S \cup
T) + f(S \cap T)$. Given a set $S \subseteq V$, we define the
\emph{gain} of an element $j \notin S$ in the context $S$ as $f(j | S)
= f(S \cup j) - f(S)$. A perhaps more intuitive
characterization of submodularity is as follows:
a function $f$ is submodular if it satisfies
\emph{diminishing marginal returns}, namely $f(j | S) \geq f(j | T)$
for all $S \subseteq T, j \notin T$, and is \emph{monotone} if $f(j |
S) \geq 0$ for all $j \notin S, S \subseteq V$.

Two central optimization problems involving submodular functions are submodular minimization and submodular maximization. Moreover, it is often natural to want to optimize these functions subject to combinatorial constraints~\cite{nemhauser1978,rkiyersemiframework2013,jegelka2011-inference-gen-graph-cuts,goel2009optimal}. 

In this paper, we shall study the problem of robust submodular optimization. Often times in applications we want to optimize several objectives (or criteria) together. There are two natural formulations of this. One is the average case, where we can optimize the (weighted) sum of the submodular functions. Examples of this have been studied in data summarization applications~\cite{lin2012learning,tsciatchek14image,gygli2015video}. The other is robust or worst case, where we want to minimize (or maximize) the maximum (equivalently minimum) among the functions. Examples of this have been proposed for sensor placement and observation selection~\cite{krause08robust}. Robust or worst case optimization is becoming increasingly important since solutions achieved by minimization and maximization can be unstable to perturbations in data. Often times submodular functions in applications are instantiated from various properties of the data (features, similarity functions, clusterings etc.) and obtaining results which are robust to perturbations and variations in this data is critical.

Given monotone submodular functions $f_1, f_2, \cdots, f_l$ to minimize, this paper studies the following optimization problem
\begin{align}
\label{robustsubmin}
\mbox{\textsc{Robust-SubMin}: }\min_{X \in \mathcal C} \max_{i = 1:l} f_i(X)
\end{align}
where $\mathcal C$ stands for combinatorial constraints, which include cardinality, matroid, spanning trees, cuts, s-t paths etc. We shall call this problems \textsc{Robust-SubMin}. A closely related problem is the maximization version of this problem called
 \textsc{Robust-SubMax} as:
 \begin{align}
\label{robustsubmax}
\mbox{\textsc{Robust-SubMax}: } \max_{X \in \mathcal C} \min_{i = 1:l} f_i(X)
\end{align}
Note that when $l = 1$, \textsc{Robust-SubMin} becomes constrained submodular minimization and \textsc{Robust-SubMax} becomes constrained submodular maximization. 

\subsection{Motivating Applications}
This section provides an overview of two specific applications which motivate \textsc{Robust-SubMin}. 
\begin{figure}
\begin{center}
\includegraphics[width=0.18\textwidth]{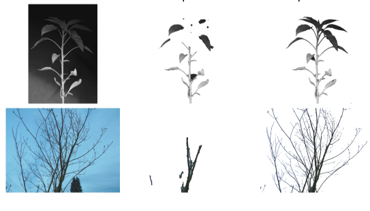}
~
\includegraphics[width=0.26\textwidth]{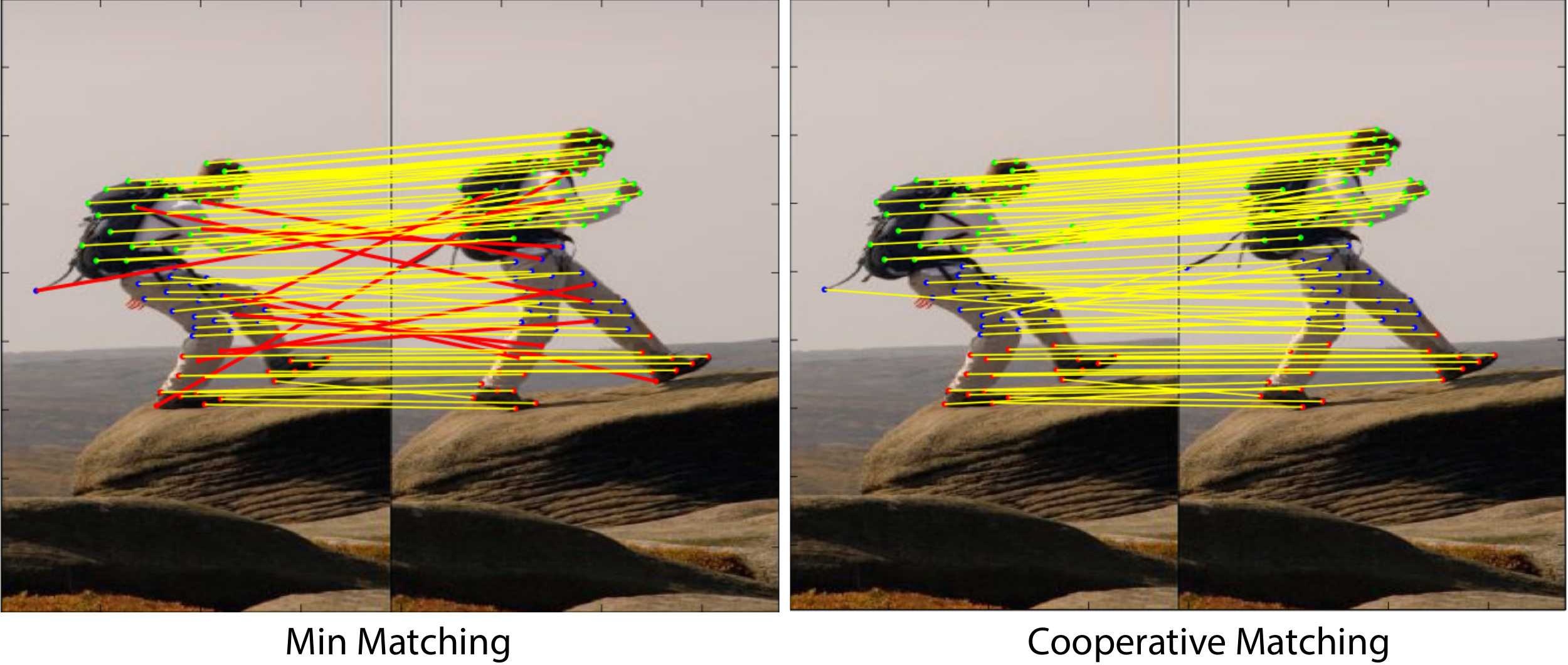}
\caption{
An illustration of co-operative cuts~\cite{jegelka2011-inference-gen-graph-cuts} and cooperative matchings~\cite{iyer2019near}.}
\end{center}
\label{fig:illustration}
\end{figure}

\noindent \textbf{Robust Co-operative Cuts: } Markov random fields with pairwise attractive potentials occur naturally in modeling image segmentation and related applications~\cite{boykov2004experimental}. While models are tractably solved using graph-cuts, they suffer from the shrinking bias problem, and images with elongated edges are not segmented properly. When modeled via a submodular function, however, the cost of a cut is not just the sum of the edge weights, but a richer function that allows cooperation between edges, and yields superior results on many challenging tasks (see, for example, the results of the image segmentations in~\cite{jegelka2011-nonsubmod-vision}). This was achieved in~\cite{jegelka2011-nonsubmod-vision} by partitioning the set of edges $\mathcal E$ of the grid graph into groups of similar edges (or types) $\mathcal E_1, \cdots, \mathcal E_k$, and defining a function $f(S) = \sum_{i = 1}^k \psi_i(w(S \cap \mathcal E_i)), S \subseteq \mathcal E$, where $\psi_i$s are concave functions and $w$ encodes the edge potentials. This ensures that we offer a \emph{discount} to edges of the same type. The left image in Figure~\ref{fig:illustration} illustrates co-operative cuts~\cite{jegelka2011-inference-gen-graph-cuts}. However, instead of taking a single clustering (i.e. a single group of edges $\mathcal E_1, \cdots, \mathcal E_k$), one can instantiate several clusterings  $\{(\mathcal E^1_1, \cdots, \mathcal E^1_k), \cdots, (\mathcal E^l_1, \cdots, \mathcal E^l_k)\}$ and define a robust objective: $f_{robust}(S) = \max_{i = 1:l} \sum_{j = 1}^k \psi_j(w(S \cap \mathcal E^i_j)), S \subseteq \mathcal E$. Minimizing $f_{robust}$ over the family of s-t cuts can achieve segmentations that are robust to many such groupings of edges thereby achieving \emph{robust co-operative segmentations}. We call this \emph{Robust Co-operative Cuts} and this becomes instance of \textsc{Robust-SubMin} with $f_{robust}$ defined above and the constraint $\mathcal C$ being the family of s-t cuts.\looseness-1

\noindent \textbf{Robust Co-operative Matchings: } The simplest model for matching key-points in pairs of images (which is also called the correspondence problem) can be posed as a bipartite matching, also called the assignment problem. These models, however, do not capture interaction between the pixels.  One kind of desirable interaction is that similar or neighboring pixels be matched together. We can achieve this as follows. First we cluster the key-points in the two images into $k$ groups and the induced clustering of edges can be given a discount via a submodular function. This model has demonstrated improved results for image correspondence problems~\cite{iyer2019near}. The right image in Figure~\ref{fig:illustration} illustrates co-operative matchings~\cite{iyer2019near}. Similar to co-operative cuts, it makes sense to be robust among the different clusterings of pixels and minimize the worst case clustering. We call this \emph{Robust Co-operative Matchings}. For more details of this model, please see the experimental section of this manuscript.\looseness-1

\subsection{Related Work}

\noindent \textbf{Submodular Minimization and Maximization: } Most Constrained forms of submodular optimization (minimization and maximization) are NP hard even when $f$ is monotone~\cite{nemhauser1978,wolsey1982analysis,goel2009optimal,jegelka2011-inference-gen-graph-cuts}. The greedy algorithm achieves a $1- 1/e$ approximation for cardinality constrained maximization and a $1/2$ approximation under matroid constraints~\cite{fisher1978analysis}. \cite{chekuri2011submodular} achieved a tight $1 - 1/e$ approximation for matroid constraints using the continuous greedy. \cite{kulik2009maximizing} later provided a similar $1 - 1/e$ approximation under multiple knapsack constraints. Constrained submodular minimization is much harder -- even with simple constraints such as a cardinality lower
bound constraints, the problems are not approximable better than a
polynomial factor of $\Omega(\sqrt{n} )$~\cite{svitkina2008submodular}. Similarly the problems of minimizing a submodular function under covering constraints~\cite{iwata2009submodular}, spanning trees, perfect matchings, paths~\cite{goel2009approximability} and cuts~\cite{jegelka2011-inference-gen-graph-cuts} have similar polynomial hardness factors. In all of these cases, matching upper bounds (i.e approximation algorithms) have been provided~\cite{svitkina2008submodular,iwata2009submodular,goel2009approximability,jegelka2011-inference-gen-graph-cuts}. In~\cite{rkiyersemiframework2013,curvaturemin}, the authors provide a scalable semi-gradient based framework and curvature based bounds which improve upon the worst case polynomial factors for functions with bounded curvature $\kappa_f$ (which several submodular functions occurring in real world applications have).\looseness-1 

\noindent \textbf{Robust Submodular Maximization: } Unlike \textsc{Robust-SubMin}, \textsc{Robust-SubMax} has been extensively studied in literature. One of the first papers to study robust submodular maximization was~\cite{krause08robust}, where the authors study \textsc{Robust-SubMax} with cardinality constraints. The authors reduce this problem to a submodular set cover problem using the \emph{saturate} trick to provide a bi-criteria approximation guarantee. \cite{anari2019robust} extend this work and study \textsc{Robust-SubMax} subject to matroid constraint. They provide bi-criteria algorithms by creating a union of $O(\log l/\epsilon)$ independent sets, with the union set having a guarantee of $1 - \epsilon$. They also discuss extensions to knapsack and multiple matroid constraints and provide bicriteria approximation of $(1 - \epsilon, O(\log l/\epsilon))$. \cite{powers2017constrained} also study the same problem. However, they take a different approach by presenting a bi-criteria algorithm that outputs a feasible set that is good only for a fraction of the $k$ submodular functions $g_i$. \cite{chen2017robust,wilder2017equilibrium} study a slightly general problem of robust non-convex optimization (of which robust submodular optimization is a special case), but they provide weaker guarantees compared to \cite{krause2008efficient,anari2019robust}. 

\noindent \textbf{Robust Min-Max Combinatorial Optimization: } From a minimization perspective, several researchers have studied robust min-max combinatorial optimization (a special case of \textsc{Robust-SubMin} with modular functions) under different combinatorial constraints (see~\cite{aissi2009min,kasperski2016robust} for a survey). Unfortunately these problems are NP hard even for constraints such as knapsack, s-t cuts, s-t paths, assignments and spanning trees where the standard linear cost problems are poly-time solvable~\cite{aissi2009min,kasperski2016robust}. Moreover, the lower bounds on hardness of approximation is $\Omega(log^{1-\epsilon} l)$ ($l$ is the number of functions) for s-t cuts, paths and assignments~\cite{kasperski2009approximability} and $\Omega(log^{1-\epsilon} n)$ for spanning trees~\cite{kasperski2011approximability} for any $\epsilon > 0$. For the case when $l$ is bounded (a constant), fully polynomial time approximation schemes have been proposed for a large class of constraints including s-t paths, knapsack, assignments and spanning trees~\cite{aissi2010general,aissi2009min,kasperski2016robust}. From an approximation algorithm perspective, the best known general result is an approximation factor of $l$ for constraints where the linear function can be exactly optimized in poly-time. For special cases such as spanning trees and shortest paths, one can achieve improved approximation factors of $O(\log n)$~\cite{kasperski2011approximability,kasperski2018approximating} and $\tilde{O}(\sqrt{n})$\footnote{Ignores $\log$ factors} ~\cite{kasperski2018approximating} respectively.\looseness-1

\subsection{Our Contributions}
While past work has mainly focused on \textsc{Robust-SubMax}, This paper is the first work to provide approximation bounds for \textsc{Robust-SubMin}. We start by first providing the hardness results. These results easily follow from the hardness bounds of constrained submodular minimization and robust min-max combinatorial optimization. Next, we provide four families of algorithms for \textsc{Robust-SubMin}. The first and simplest approach approximates the $\max$ among the functions $f_1, \cdots, f_l$ in \textsc{Robust-SubMin} with the average. This in turn converts \textsc{Robust-SubMin} into a constrained submodular minimization problem. We show that solving the constrained submodular minimization with the average of the functions still yields an approximation factor for \textsc{Robust-SubMin} though it is worse compared to the hardness by a factor of $l$. While this approach is conceptually very simple, we do not expect this perform well in practice. We next study the Majorization-Minimization family of algorithms, where we sequentially approximate the functions $f_1, \cdots, f_l$ with their modular upper bounds. Each resulting sub-problem involves solving a min-max combinatorial optimization problem and we use specialized solvers for the different constraints. The third algorithm is the Ellipsoidal Approximation (EA) where we replace the functions $f_i$'s with their ellipsoidal approximations. We show that this achieves the tightest approximation factor for several constraints, but is slower to the majorization-minimization approach in practice. The fourth technique is a continuous relaxation approach where we relax the discrete problem into a continuous one via the \lovasz{} extension. We show that the resulting robust problem becomes a convex optimization problem with an appropriate recast of the constraints. We then provide a rounding scheme thereby resulting in approximation factors for \textsc{Robust-SubMin}. Tables 1 and 2 show the hardness and the resulting approximation factors for several important combinatorial constraints. Along the way, we also provide a family of approximation algorithms and heuristics for \textsc{Robust-Min} (i.e. with modular functions $f_i$ under several combinatorial constraints. We finally compare the four algorithms for different constraints and in each case, discuss heuristics which will make the resulting algorithms practical and scalable for real world problems. Finally, we empirically show the performance of our algorithms on synthetic and real world datasets. We demonstrate the utility of our models in robust co-operative matching and show that our robust model outperforms simple co-operative model from~\cite{iyer2019near} for this task.

\begin{table*}
       \begin{center}
       \scriptsize{
           \begin{tabular}{|l|c|c|c|c|c|c|}
  \hline
    Constraint & Hardness & MMin-AA & EA-AA & MMin & EA & CR
        \\ \hline
    Cardinality ($k$) & $K(\sqrt{n}, \kappa)$ & $l K(k, \kappa)$ & $l K(\sqrt{n}\log n, \kappa)$ & $O(\log l K(k, \kappa_{wc})/\log\log l)$ & $\tilde{O}(\sqrt{m \log l/\log \log l})$ & $n-k+1$ \\
    Trees & $M(K(n, \kappa), \log n)$ & $l K(n, \kappa)$ & $l K(\sqrt{m}\log m, \kappa)$ & $O(\min(\log n, l) K(n, \kappa_{wc}))$ &  $O(\min(\sqrt{\log n}, \sqrt{l}) \sqrt{m}\log m)$ & $m-n+1$ \\
    Matching & $M(K(n, \kappa), \log l)$ & $l K(n, \kappa)$ & $l K(\sqrt{m}\log m, \kappa)$ & $l K(n, \kappa)$ & $\tilde{O}(\sqrt{lm})$ & $n$ \\
    s-t Cuts & $M(K(\sqrt{m}, \kappa), \log l)$ & $l K(n, \kappa)$ & $l K(\sqrt{m}\log m, \kappa)$ & $l K(n, \kappa)$ & $\tilde{O}(\sqrt{lm})$ & $n$ \\
    s-t Paths & $M(K(\sqrt{m}, \kappa), \log l)$ & $l K(n, \kappa)$ & $l K(\sqrt{m}\log m, \kappa)$ & $O(\min(\sqrt{n}, l) K(n, \kappa_{wc}))$ & $O(\min(n^{0.25}, \sqrt{l})\sqrt{m}\log m)$ & $m$ \\
    Edge Cov. & $K(n, \kappa)$ & $l K(n, \kappa)$ & $l K(\sqrt{m}\log m, \kappa)$ & $l K(n, \kappa)$ & $\tilde{O}(\sqrt{lm})$ & $n$ \\
    Vertex Cov. & $2$ & $l K(n, \kappa)$ & $l K(\sqrt{n}\log n, \kappa)$ & $l K(n, \kappa)$ & $\tilde{O}(\sqrt{ln})$ & $2$ \\
   \hline
  \end{tabular}
    \label{tab:1}
    \caption{Approximation bounds and Hardness for \textsc{Robust-SubMin}. $M(.)$ stands for $\max(.)$}
    }
\end{center}
\end{table*}

\begin{table*}
\begin{center}
\begin{tabular}{|l|c|c|c|}
  \hline
    Constraint & Hardness & MMin & EA
        \\ \hline
    Knapsack & $K(\sqrt{n}, \kappa)$ & $K(n, \kappa)$ & $O(K(\sqrt{n}\log n, \kappa))$ \\
    Trees & $K(n, \kappa)$ & $K(n, \kappa)$ & $O(K(\sqrt{m}\log m, \kappa))$ \\
    Matchings & $K(n, \kappa)$ & $K(n, \kappa)$ & $O(K(\sqrt{m}\log m, \kappa))$ \\
    s-t Paths & $K(n, \kappa)$ & $K(n, \kappa)$ & $O(K(\sqrt{m}\log m, \kappa))$ \\
   \hline
  \end{tabular}
    \label{tab:2}
    \caption{Approximation bounds and Hardness of in \textsc{Robust-SubMin} with $l$ constant}
\end{center}
\end{table*}

\section{Main Ideas and Techniques}
In this section, we will review some of the constructs and techniques used in this paper to provide approximation algorithms for \textsc{Robust-SubMin}.

\noindent \textbf{Curvature: } Given a submodular function $f$, the curvature~\cite{vondrak2010submodularity, conforti1984submodular, curvaturemin}: 
\begin{align*}
\kappa_f = 1 - \min_{j \in V} \frac{f(j | V \backslash j)}{f(j)}
\end{align*}
Thanks to submodularity, it is easy to see that $0 \leq \kappa_f \leq 1$. We define a quantity $K(f, \kappa) = f/(1 + (1 - \kappa)(f - 1))$, where $\kappa$ is the curvature. Note that $K(f, \kappa)$ interpolates between $K(f, 1) = f$ and $K(f, 0) = 1$. This quantity shall repeatedly come up in the approximation and hardness bounds in this paper.

\noindent \textbf{The Submodular Polyhedron and \lovasz{} extension }
For a submodular function $f$, the submodular polyhedron $\mathcal P_f$ and the corresponding base polytope $\mathcal B_f$ are respectively defined as $\mathcal P_f = \{ x : x(S) \leq f(S), \forall S \subseteq V \} 
\;\;\;
\mathcal B_f = \mathcal P_f \cap \{ x : x(V) = f(V) \}$. 
For a vector $x \in \mathbb{R}^V$ and a set $X \subseteq V$, we write $x(X)
= \sum_{j \in X} x(j)$. Though $\mathcal P_f$ is defined via $2^n$ inequalities, its extreme point can be easily characterized~\cite{fujishige2005submodular,edmondspolyhedra}. Given any permutation $\sigma$ of the ground set $\{1, 2, \cdots, n\}$, and an associated chain $\emptyset = S^{\sigma}_0 \subseteq S^{\sigma}_1 \subseteq \cdots \subseteq S^{\sigma}_n = V$ with $S^{\sigma}_i =
\{ \sigma(1), \sigma(2), \dots, \sigma(i) \}$, a vector $h^f_{\sigma}$ satisfying,
$h^f_{\sigma}(\sigma(i) = f(S^{\sigma}_i) - f(S^{\sigma}_{i-1}) = f(\sigma(i) | S^{\sigma}_{i-1}), \forall i = 1, \cdots, n$ forms an extreme point of $\mathcal P_f$. Moreover, a natural convex extension of a submodular function, called the \lovasz{} extension~\cite{lovasz1983,edmondspolyhedra} is closely related to the submodular polyhedron, and is defined as $\hat{f}(x) = \max_{h \in \mathcal P_f} \langle h, x \rangle$. Thanks to the properties of the polyhedron, $\hat{f}(x)$ can be efficiently computed: Denote $\sigma_x$ as an ordering induced by $x$, such that $x(\sigma_x(1)) \geq x(\sigma_x(2)) \geq \cdots x(\sigma_x(n))$. Then the \lovasz{} extension is $\hat{f}(x) = \langle h^f_{\sigma_x}, x \rangle$~\cite{lovasz1983,edmondspolyhedra}. The gradient of the \lovasz{} extension $\nabla \hat{f}(x) = h^f_{\sigma_x}$.

\noindent \textbf{Modular lower bounds (Sub-gradients): } Akin to convex functions, submodular functions have tight modular lower bounds. These bounds are related to the sub-differential $\partial_f(Y)$ of the submodular set function $f$ at a set $Y \subseteq V$, which is defined 
\cite{fujishige2005submodular}
as: $\partial_f(Y) = \{y \in \mathbb{R}^n: f(X) - y(X) \geq f(Y) - y(Y),\; \text{for all } X \subseteq V\}$. Denote a sub-gradient at $Y$ by $h_Y \in \partial_f(Y)$. Define $h_Y = h^f_{\sigma_Y}$ (see the definition of $h^f_{\sigma}$ from the previous paragraph) forms a lower bound of $f$, tight at $Y$ --- i.e.,
$h_Y(X) = \sum_{j \in X} h_Y(j) \leq f(X), \forall X
\subseteq V$ and $h_Y(Y) = f(Y)$. Notice that the extreme points of a sub-differential are a subset of the extreme points of the submodular polyhedron. 

\noindent \textbf{Modular upper bounds (Super-gradients): }
We can also define super-differentials $\partial^f(Y)$ of a submodular
function 
\cite{jegelka2011-nonsubmod-vision,rkiyersubmodBregman2012}
at
$Y$: $\partial^f(Y) = \{y \in \mathbb{R}^n: f(X) - y(X) \leq f(Y) - y(Y); \text{for all } X \subseteq V\}$. It is possible, moreover, to provide specific super-gradients~\cite{rkiyersubmodBregman2012,rkiyersemiframework2013, iyer2015polyhedral, iyermirrordescent, iyer2015submodular} that define the following two modular upper bounds:
\begin{align}
m^f_{X, 1}(Y) \triangleq f(X) - \sum_{j \in X \backslash Y } f(j| X \backslash j) + \sum_{j \in Y \backslash X} f(j| \emptyset), \\
m^f_{X, 2}(Y) \triangleq f(X) - \sum_{j \in X \backslash Y } f(j| V \backslash j) + \sum_{j \in Y \backslash X} f(j| X).
\end{align}

Then $m^f_{X, 1}(Y) \geq f(Y)$ and $m^f_{X, 2}(Y) \geq f(Y), \forall Y \subseteq V$ and $m^f_{X, 1}(X) = m^f_{X, 2}(X) = f(X)$. Also note that $m^f_{\emptyset, 1}(Y) = m^f_{\emptyset, 2}(Y) = \sum_{j \in Y} f(j | \emptyset)$. For simplicity denote this as $m^f_{\emptyset}(X)$. Then the following result holds:
\begin{lemma}\cite{rkiyersemiframework2013, curvaturemin}
Given a submodular function $f$ with curvature $\kappa_f$, $f(X) \leq m^f_{\emptyset}(X) \leq K(|X|, \kappa_f) f(X)$
\end{lemma}
\noindent \textbf{Ellipsoidal Approximation: } Another generic approximation of a submodular function, introduced by Goemans et.\ al~\cite{goemans2009approximating}, is based on approximating the submodular polyhedron by an ellipsoid. The main result states that any polymatroid (monotone submodular) function $f$,
can be approximated by a function of the form $\sqrt{w^f(X)}$ for a
certain modular weight vector $w^f \in \mathbb R^V$, such that $\sqrt{w^f(X)} \leq f(X) \leq
O(\sqrt{n}\log{n}) \sqrt{w^f(X)}, \forall X \subseteq V$. A simple trick then provides a curvature-dependent approximation~\cite{curvaturemin} ---
we define the $\curvf{f}$-\emph{\truncated{}} version of $f$ as
follows: $f^{\kappa}(X) \triangleq \bigl[f(X) - {(1 - \curvf{f})} \sum_{j \in X} f(j)\bigr]/ \curvf{f}$. Then, the submodular function $f^{\text{ea}}(X) = \curvf{f} \sqrt{w^{f^{\kappa}}(X)} + (1 -
  \curvf{f})\sum_{j \in X} f(j)$ satisfies~\cite{curvaturemin}:
\begin{align}
f^{\text{ea}}(X) 
\leq f(X) 
\leq O\left(K(\sqrt{n}\log n, \kappa_f)\right) f^{\text{ea}}(X), \forall X \subseteq V
\end{align}

\section{Algorithms and Hardness Results}
In this section, we shall go over the hardness and approximation algorithms for \textsc{Robust-SubMin}. We shall consider two cases, one where $l$ is bounded (i.e. its a constant), and the other where $l$ is unbounded. We start with some notation. We denote the graph as $G = (\mathcal V, \mathcal E)$ with $|\mathcal V| = n, |\mathcal E| = m$.  Depending on the problem at hand, the ground set $V$ can either be the set of edges ($V = \mathcal E$) or the set of vertives ($V = \mathcal V$). The groundset is the set of edges in the case of trees, matchings, cuts, paths and edge covers, while in the case of vertex covers, they are defined on the vertices.

\subsection{Hardness of \textsc{Robust-SubMin}} 
Since \textsc{Robust-SubMin} generalizes robust min-max combinatorial optimization (when the functions are modular), we have the hardness bounds from~\cite{kasperski2009approximability,kasperski2011approximability}. For the modular case, the lower bounds are $\Omega(log^{1-\epsilon} l)$ ($l$ is the number of functions) for s-t cuts, paths and assignments~\cite{kasperski2009approximability} and $\Omega(log^{1-\epsilon} n)$ for spanning trees~\cite{kasperski2011approximability} for any $\epsilon > 0$. These hold unless $NP \subseteq DTIME(n^{poly \log n})$\cite{kasperski2009approximability,kasperski2011approximability}. Moreover, since \textsc{Robust-SubMin} also generalizes constrained submodular minimization, we get the curvature based hardness bounds from~\cite{curvaturemin,goel2009optimal,jegelka2011-inference-gen-graph-cuts}. The hardness results are in the first column of Table~1. The curvature $\kappa$ corresponds to the worst curvature among the functions $f_i$ (i.e. $\kappa = \max_i \kappa_i$). 

\subsection{Algorithms with Modular functions $f_i$'s}
In this section, we shall study approximation algorithms for \textsc{Robust-SubMin} when the functions $f_i$'s are modular, i.e. $f_i(X) = \sum_{j \in X} f_i(j)$. We call this problem \textsc{Robust-Min}. Most previous work~\cite{aissi2010general,aissi2009min,kasperski2016robust} has focused on fully polynomial approximation schemes when $l$ is small. These algorithms are exponential in $l$ and do not scale beyond $l = 3$ or $4$. Instead, we shall study approximation algorithms for this problem. Define two simple approximations of the function $F(X) = \max_i f_i(X)$ when $f_i$'s are modular. The first is $\hat{F}(X) = \sum_{i \in X} \max_{j = 1:l} f_j(i)$ and the second is $\tilde{F}(X) = 1/l \sum_{i = 1}^l \sum_{j \in X} f_i(j) = 1/l \sum_{i = 1}^l f_i(X)$. 

 \begin{lemma}\label{mod-approximations}
 Given $f_i(j) \geq 0$, it holds that $\hat{F}(X) \geq F(X) \geq \frac{1}{l} \hat{F}(X)$. Furthermore, $\tilde{F}(X) \leq F(X) \leq l\tilde{F}(X)$. Given a $\beta$-approximate algorithm for optimizing linear cost functions over the constraint $\mathcal C$, we can achieve an $l\beta$-approximation algorithm for \textsc{Robust-Min}.\looseness-1
 \end{lemma}
 \begin{proof}
We first prove the second part.  Its easy to see that $F(X) = \max_i f_i(X) \leq \sum_i f_i(X) = l\tilde{F}(X)$ which is the second inequality. The first inequality follows from the fact that $\sum_i f_i(X) \leq l \max_i f_i(X)$. To prove the first result, we start with proving $F(X) \leq \hat{F}(X)$. For a given set $X$, let $i_X$ be the index which maximizes $F$ so $F(X) = \sum_{j \in X} f_{i_X}(j)$. Then $f_{i_X}(j) \leq \max_i f_i(j)$ from which we get the result. Next, observe that $\hat{F}(X) \leq \sum_{i = 1}^l f_i(X) = l\tilde{F}(X) \leq lF(X)$ which follows from the inequality corresponding to $\tilde{F}$. 

Now given a $\beta$-approximation algorithm for optimizing linear cost functions over the constraint $\mathcal C$, denote $\tilde{X}$ by optimizing $\tilde{F}(X)$ over $\mathcal C$. Also denote $X^*$ as the optimal solution by optimizing $F$ over $\mathcal C$ and $\tilde{X^*}$ be the optimal solution
for optimizing $\tilde{F}$ over $\mathcal C$. Since $\tilde{X}$ is a $\beta$ approximation for $\tilde{F}$ over $\mathcal C$, $\tilde{F}(\tilde{X}) \leq \beta \tilde{F}(\tilde{X^*}) \leq \beta \tilde{F}(X^*)$. The last inequality holds since $X^*$ is feasible (i.e. it belongs to $\mathcal C$) so it must hold that $\tilde{F}(X^*) \geq \tilde{F}(\tilde{X^*})$.
A symmetric argument applies to $\hat{F}$. In particular, $F(\hat{X}) \leq \hat{F}(\hat{X}) \leq \beta \hat{F}(X^*) \leq \beta l F(X^*)$. The inequality $\hat{F}(\hat{X}) \leq \beta \hat{F}(X^*)$ again holds due to an argument similar to the $\tilde{F}$ case.
 \end{proof}
 
 Note that for most constraints $\mathcal C$, $\beta$ is $1$. In practice, we can take the better of the two solutions obtained from the two approximations above.

Next, we shall look at higher order approximations of $F$. Define the function $F_a(X) = (\sum_{i = 1}^l [f_i(X)]^a)^\frac{1}{a}$. Its easy to see that $F_a(X)$ comes close to $F$ as $a$ becomes large. 
\begin{lemma}
Define $F_a(X) = (\sum_{i = 1}^l [f_i(X)]^a)^\frac{1}{a}$. Then it holds that $F(X) \leq F_a(X) \leq l^\frac{1}{a} F(X)$. Moreover, given a $\beta$-approximation algorithm for optimizing $F_a(X)$ over $\mathcal C$, we can obtain a $\beta l^\frac{1}{a}$-approximation for \textsc{Robust-Min}.
\end{lemma}
\arxiv{
\begin{proof}
Its easy to see that $F(X) \leq F_a(X)$ since $[F(X)]^a \leq \sum_{i = 1}^l [f_i(X)]^a$. Next, note that $\forall i, f_i(X) \leq F(X)$ and hence $\sum_{i = 1}^l [f_i(X)]^a \leq l [F(X)]^a$ which implies that $F_a(X) \leq k^{\frac{1}{a}} F(X)$ which proves the result. The second part follows from arguments similar to the Proof of Lemma 2.
\end{proof}
}

Note that optimizing $F_a(X)$ is equivalent to optimizing $\sum_{i = 1}^l [f_i(X)]^a$ which is a higher order polynomial function. For example, when $a = 2$ we get the quadratic combinatorial program~\cite{buchheim2018quadratic} which can possibly result in a a $\sqrt{l}$ approximation given a exact or approximate quadratic optimizer. Unfortunately, optimizing this problem is NP hard for most constraints when $a \geq 2$. However, several heuristics and efficient solvers~\cite{benson1999mixed,lawler1963quadratic,buchheim2018quadratic,loiola2007survey} exist for solving this in the quadratic setting. Moreover, for special cases such as the assignment (matching) problem, specialized algorithms such as the Graduated assignment algorithm~\cite{gold1996graduated} exist for solving this. While these are not guaranteed to theoretically achieve the optimal solution, they work quite in practice, thus yielding a family of heuristics for \textsc{Robust-Min}.

\subsection{Average Approximation based Algorithms} 
In this section, we shall look at a simple approximation of $\max_i f_i(X)$, which in turn shall lead to an approximation algorithm for \textsc{Robust-SubMin}. We first show that $f_{avg}(X) = 1/l \sum_{i = 1}^l f_i(X)$ is an $l$-approximation of $\max_{i = 1:l} f_i(X)$, which implies that minimizing $f_{avg}(X)$ implies an approximation for \textsc{Robust-SubMin}.
\begin{theorem}
Given a non-negative set function $f$, define $f_{avg}(X) = \frac{1}{l}  \sum_{i = 1}^l f_i(X)$. Then $f_{avg}(X) \leq \max_{i = 1:l} f_i(X) \leq lf_{avg}(X)$. Denote $\hat{X}$ as $\beta$-approximate optimizer of $f_{avg}$. Then $\max_{i = 1:l} f_i(\hat{X}) \leq l\beta \max_{i = 1:l} f_i(X^*)$ where $X^*$ is the exact minimizer of \textsc{Robust-SubMin}.
\end{theorem}
\arxiv{
 \begin{proof}
 To prove the first part, notice that $f_i(X) \leq \max_{i = 1:l} f_i(X)$, and hence $1/l \sum_i f_i(X) \leq \max_{i = 1:l} f_i(X)$. The other inequality also directly follows since the $f_i$'s are non-negative and hence $\max_{i = 1:l} f_i(X) \leq \sum_i f_i(X) = l f_{avg}(X)$. To prove the second part, observe that $\max_i f_i(\hat{X}) \leq lf_{avg}(\hat{X}) \leq l\beta f_{avg}(X^*) \leq l\beta \max_i f_i(X^*)$. The first inequality holds from the first part of this result, the second inequality holds since $\hat{X}$ is a $\beta$-approximate optimizer of $f_{avg}$ and the third part of the theorem holds again from the first part of this result.
 \end{proof}}
 
 Since $f_{avg}$ is a submodular function, we can use the majorization-minimization (which we call MMin-AA) and ellipsoidal approximation (EA-AA) for constrained submodular minimization~\cite{rkiyersemiframework2013,curvaturemin,goel2009optimal,jegelka2011-inference-gen-graph-cuts}. The following corollary provides the approximation guarantee of MMin-AA and EA-AA for \textsc{Robust-SubMin}.
 \begin{corollary}
 Using the majorization minimization (MMin) scheme with the average approximation achieves an approximation guarantee of $l K(|X^*|, \kappa_{avg})$ where $X^*$ is the optimal solution of \textsc{Robust-SubMin} and $\kappa_{avg}$ is the curvature of $f_{avg}$. Using the curvature-normalized ellipsoidal approximation algorithm from~\cite{curvaturemin,goemans2009approximating} achieves a guarantee of $O(lK(|\mathcal V|\log |\mathcal V|, \kappa_{avg}))$
 \end{corollary}
This corollary directly follows by combining the approximation guarantee of MMin and EA~\cite{curvaturemin,rkiyersemiframework2013} with Theorem 1. Substituting the values of $|\mathcal V|$ and $|X^*|$ for various constraints, we get the results in Table~1. While the average case approximation method provides a bounded approximation guarantee for \textsc{Robust-SubMin}, it defeats the purpose of the robust formulation. Moreover, the approximation factor is worse by a factor of $l$. Below, we shall study some techniques which directly try to optimize the robust formulation.
 
 \subsection{Majorization-Minimization Algorithm} The Majorization-Minimization algorithm is a sequential procedure which uses upper bounds of the submodular functions defined via supergradients. Starting with $X^0 = \emptyset$, the algorithm proceeds as follows. At iteration $t$, it constructs modular upper bounds for each function $f_i$, $m^{f_i}_{X^t}$ which is tight at $X^t$. We can use either one of the two modular upper bounds defined in Section 2. The set $X^{t+1} = \mbox{argmin}_{X \in \mathcal C} \max_i m^{f_i}_{X^t}(X)$. This is a min-max robust optimization problem. The following theorem provides the approximation guarantee for MMin.\looseness-1
 \begin{theorem} \label{thm3}
 If $l$ is a constant, MMin achieves an approximation guarantee of $(1 + \epsilon)K(|X^*|, \kappa_{wc})$ for the knapsack, spanning trees, matching and s-t path problems. The complexity of this algorithm is exponential in $l$. When $l$ is unbounded, MMin achieves an approximation guarantee of $lK(|X^*|, \kappa_{wc})$. For spanning trees and shortest path constraints, MMin achieves a $O(\min(\log n, l) K(n, \kappa_{wc}))$ and a $O(\min(\sqrt{n}, l) K(n, \kappa_{wc}))$ approximation. Under cardinality and partition matroid constraints, MMin achieves a $O(\log l K(n, \kappa_{wc})/\log\log l)$ approximation.
 \end{theorem}
 
 Substituting the appropriate bounds on $|X^*|$ for the various constraints, we get the results in Tables 1 and 2. $\kappa_{wc}$ corresponds to the worst case curvature $\max_i \kappa_{f_i}$. 
 
 We now elaborate on the Majorization-Minimization algorithm. 
  At every round of the majorization-minimization algorithm we need to solve 
 \begin{align} \label{mmin-subproblem}
 X^{t+1} = \mbox{argmin}_{X \in \mathcal C} \max_i m^{f_i}_{X^t}(X).
 \end{align}
 We consider three cases. The first is when $l$ is a constant. In that case, we can use an FPTAS to solve Eq.~\eqref{mmin-subproblem}~\cite{aissi2009min,aissi2010general}. We can obtain a $1+\epsilon$ approximation with complexities of $O(n^{l+1}/\epsilon^{l-1})$ for shortest paths, $O(mn^{l+4}/\epsilon^l \log{n/\epsilon})$ for spanning trees, $O(mn^{l+4}/\epsilon^l \log{n/\epsilon})$ for matchings and  $O(n^{l+1}/\epsilon^l)$ for knapsack~\cite{aissi2010general}. The results for constant $l$ is shown in Table~2 (column corresponding to MMin). The second case is a generic algorithm when $l$ is not constant. Note that we cannot use the FPTAS since they are all exponential in $l$.  In this case, at every iteration of MMin can use the framework of approximations discussed in Section 3.2, and choose the solution with a better objective value. In particular, if we use the two modular bounds of the $\max$ function (i.e. the average the the max bounds), we can obtain $l$-approximations of $\max_i m^{f_i}_{X^t}$. One can also use the quadratic and cubic bounds which can possibly provide $\sqrt{l}$ or $l^{0.3}$ bounds. While these higher order bounds are still theoretically NP hard, there exist practically efficient solvers for various constraints (for e.g. quadratic assignment model for matchings~\cite{loiola2007survey}). Finally, for the special cases of spanning trees, shortest paths, cardinality and partition matroid constraints, there exist LP relaxation based algorithms which achieve approximation factors of $O(\log n)$, $O(\sqrt{n\log l/\log\log l}$, $O(\log l/\log \log l)$ and $O(\log l/\log \log l)$ respectively. The approximation guarantees of MMin for unbounded $l$ is shown in Table~\ref{tab:1}.

  \arxiv{
We now prove Theorem~\ref{thm3}.
 \begin{proof}
 Assume we have an $\alpha$ approximation algorithm for solving problem~\eqref{mmin-subproblem}.
 We start MMin with $X^0 = \emptyset$. We prove the bound for MMin for the first iteration. Observe that  $m^{f_i}_{\emptyset}(X)$ approximate the submodular functions $f_i(X)$ up to a factor of $K(|X|, \kappa_i)$~\cite{curvaturemin}. If $\kappa_{wc}$ is the maximum curvature among the functions $f_i$, this means that $m^{f_i}_{\emptyset}(X)$ approximate the submodular functions $f_i(X)$ up to a factor of $K(|X|, \kappa_{wc})$ as well. Hence $\max_i m^{f_i}_{\emptyset}(X)$ approximates $\max_i f_i(X)$ with a factor of $K(|X|, \kappa_{wc})$. In other words, $\max_i f_i(X) \leq \max_i m^{f_i}_{\emptyset}(X) \leq K(|X|, \kappa_{wc}) \max_i f_i(X)$. Let $\hat{X_1}$ be the solution obtained by optimizing  $m^{f_i}_{\emptyset}$ (using an $\alpha$-approximation algorithms for the three cases described above). It holds that $\max_i m^{f_i}_{\emptyset}(\hat{X}_1) \leq \alpha \max_i m^{f_i}_{\emptyset}(X^m_1)$ where $X^m_1$ is the optimal solution of $\max_i m^{f_i}_{\emptyset}(X)$ over the constraint $\mathcal C$. Furthermore, denote $X^*$ as the optimal solution of $\max_i f_i(X)$ over $\mathcal C$. Then $\max_i m^{f_i}_{\emptyset}(X^m_1) \leq \max_i m^{f_i}_{\emptyset}(X^*) \leq K(|X^*|, \kappa_{wc})\max_i f_i(X^*)$. Combining both, we see that $\max_i f(\hat{X}_1) \leq \max_i m^{f_i}_{\emptyset}(\hat{X}_1) \leq \alpha K(|X^*|, \kappa_{wc})\max_i f_i(X^*)$. We then run MMin for more iterations and only continue if the objective value increases in the next round. Using the values of $\alpha$ for the different cases above, we get the results.
 \end{proof}
 }

\begin{table*}
\begin{center}
\begin{tabular}{ | c |  c |  }
\hline
 Constraints & $\hat{\mathcal P_{\mathcal C}}$  \\ \hline
 Matroids (includes Spanning Trees, Cardinality) & $\{x \in [0, 1]^n, x(S) \geq
r_{\mathcal M}(V) - r_{\mathcal M}(V \backslash S), \forall S
\subseteq V\}$ \\ \hline
 Set Covers (includes Vertex Covers and Edge Covers) & $\{x \in [0, 1]^{|\mathcal
  S|} \mid \sum_{i: u \in S_i} x(i) \geq c_u, \forall u \in \mathcal U\}$ \\ \hline
  s-t Paths & $\{x \in [0, 1]^{|\mathcal E|} \mid \sum_{e \in C}x(e) \geq 1$ , for every
s-t cut $C \subseteq \mathcal E\}$ \\ \hline
s-t Cuts & $\{x \in [0, 1]^{|\mathcal E|} \mid \sum_{e \in P}x(e) \geq 1$,
for every s-t path $P \subseteq \mathcal E\}$ \\ \hline
\end{tabular}
\end{center}
\caption{The Extended Polytope $\hat{\mathcal P_{\mathcal C}}$ for many of the combinatorial constraints discussed in this paper. See~\cite{cvxframework} for more details.}
\end{table*}

\subsection{Ellipsoidal Approximation Based Algorithm} 
Next, we use the Ellipsoidal Approximation to approximate the submodular function $f_i$. To account for the curvature of the individual functions $f_i$'s, we use the curve-normalized Ellipsoidal Approximation~\cite{curvaturemin}. We then obtain the functions $\hat{f}_i(X)$ which are of the form $(1 - \kappa_{f_i})\sqrt{w_{f_i}(X)} + \kappa_{f_i}\sum_{j \in X} f_i(j)$, and the problem is then to optimize $\max_i \hat{f}_i(X)$ subject to the constraints $\mathcal C$. This is no longer a min-max optimization problem. The following result shows that we can still achieve approximation guarantees in this case.
\begin{theorem}
For the case when $l$ is a constant, EA achieves an approximation guarantee of $O(K(\sqrt{|\mathcal V|\log |\mathcal V|}, \kappa_{wc}))$ for the knapsack, spanning trees, matching and s-t path problems. The complexity of this algorithm is exponential in $l$. When $l$ is unbounded, the EA algorithm achieves an approximation guarantee of $O(\sqrt{l}\sqrt{|\mathcal V|} \log |\mathcal V|)$ for all constraints. In the case of spanning trees, shortest paths, the EA achieves approximation factors of $O(\min(\sqrt{\log n}, \sqrt{l}) \sqrt{m}\log m)$, and $O(\min(n^{0.25}, \sqrt{l})\sqrt{m}\log m)$. Under cardinality and partition matroid constraints, EA achieves a $O(\sqrt{\log l/\log\log l} \sqrt{n}\log n)$ approximation.
\end{theorem}
For the case when $l$ is bounded, we reduce the optimization problem after the Ellipsoidal Approximation into a multi-objective optimization problem, which provides an FPTAS for knapsack, spanning trees, matching and s-t path problems~\cite{papadimitriou2000approximability,mittal2013general}. When $l$ is unbounded, we further reduce the EA approximation objective into a linear objective which then provides the approximation guarantees similar to MMin above. However, as a result, we loose the curvature based guarantee. \narxiv{A detailed proof is in the extended version of this paper. }
\arxiv{
\begin{proof}
First we start with the case when $l$ is a constant. Observe that the optimization problem is 
\begin{align*}
\min_{X \in \mathcal C} \max_i \hat{f}_i(X) = \min_{X \in \mathcal C} \max_i (1 - \kappa_{f_i})\sqrt{w_{f_i}(X)} + \\ \kappa_{f_i}\sum_{j \in X} f_i(j)
\end{align*}
This is of the form $\min_{X \in \mathcal C} \max_i \sqrt{w^i_1(X)} + w^i_2(X)$. Define $h(y^1_1, y^1_2, y^2_2, y^2_2, \cdots, y^l_1, y^l_2) = \max_i \sqrt{y^i_1} + y^i_2$. Note that the optimization problem is $\min_{X \in \mathcal C} h(w^1_1(X), w^1_2(X), \cdots, w^l_1(X), w^l_2(X))$. Observe that $h(\mathbf{y}) \leq h(\mathbf{y^{\prime}})$ if $\mathbf{y} \leq \mathbf{y^{\prime}}$. Furthermore, note that $\mathbf{y} \geq 0$. Then given a $\lambda > 1$, $h(\lambda \mathbf{y}) = \max_i \sqrt{\lambda y^i_1} + \lambda y^i_2 \leq \lambda \sqrt{y^i_1} + \lambda y^i_2 \leq \lambda h(\mathbf{y})$. As a result, we can use Theorem 3.3 from~\cite{mittal2013general} which provides an FPTAS as long as the following exact problem can be solved on $\mathcal C$: Given a constant $C$ and a vector $c \in \mathbf{R}^n$, does there exist a $x$ such that $\langle c, x \rangle = C$? A number of constraints including matchings, knapsacks, s-t paths and spanning trees satisfy this~\cite{papadimitriou2000approximability}. For these constraints, we can obtain a $1 + \epsilon$ approximation algorithm in complexity exponential in $l$. 

When $l$ is unbounded, we directly use the Ellipsoidal Approximation and the problem then is to optimize $\min_{X \in \mathcal C} \max_i \sqrt{w_{f_i}(X)}$. We then transform this to the following optimization problem: $\min_{X \in \mathcal C} \max_i w_{f_i}(X)$. Assume we can obtain an $\alpha$ approximation to the problem $\min_{X \in \mathcal C} \max_i w_{f_i}(X)$. This means we can achieve a solution $\hat{X}$ such that $\max_i w_{f_i}(\hat{X}) \leq \alpha \max_i w_{f_i}(X^{ea})$ where $X^{ea}$ is the optimal solution for the problem $\min_{X \in \mathcal C} \max_i w_{f_i}(X)$. Then observe that $\max_i \sqrt{w_{f_i}(X^{ea})} \leq \max_i \sqrt{w_{f_i}(X^*)} \leq \max_i f_i(X^*)$. Combining all the inequalities and also using the bound of the Ellipsoidal Approximation, we have $\max_i f_i(\hat{X}) \leq \beta \max_i \sqrt{w_{f_i}(\hat{X})} \leq \beta \sqrt{\alpha} \max_i \sqrt{w_{f_i}(X^{ea})} \leq \beta \sqrt{\alpha} \max_i \sqrt{w_{f_i}(X^*)} \leq \beta \sqrt{\alpha} \max_i f_i(X^*)$ where $\beta$ is the approximation of the Ellipsoidal Approximation.

We now use this result to prove the theorem. Consider two cases. First, we optimize the \emph{avg} and \emph{max} versions of $\max_i w_{f_i}(X)$ which provide $\alpha = l$ approximation. Secondly, for the special cases of spanning trees, shortest paths, cardinality and partition matroid constraints, there exist LP relaxation based algorithms which achieve approximation factors $\alpha$ being $O(\log n)$, $O(\sqrt{n\log l/\log\log l}$, $O(\log l/\log \log l)$ and $O(\log l/\log \log l)$ respectively. Substitute these values of $\alpha$ and using the fact that $\beta = O(\sqrt{|V|}\log |V|)$, we get the approximation bound.
\end{proof}
}

\begin{figure*}
    \centering
    \includegraphics[width = 0.30\textwidth]{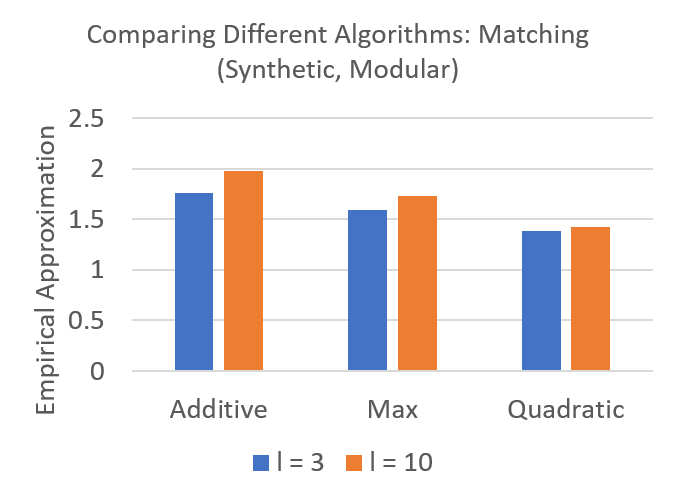} 
    \includegraphics[width = 0.32\textwidth]{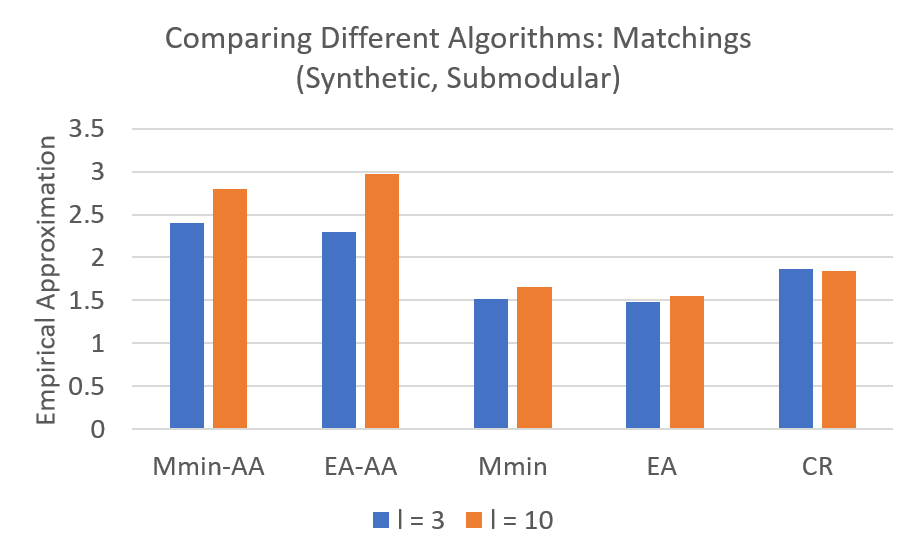}   
    \includegraphics[width = 0.32\textwidth]{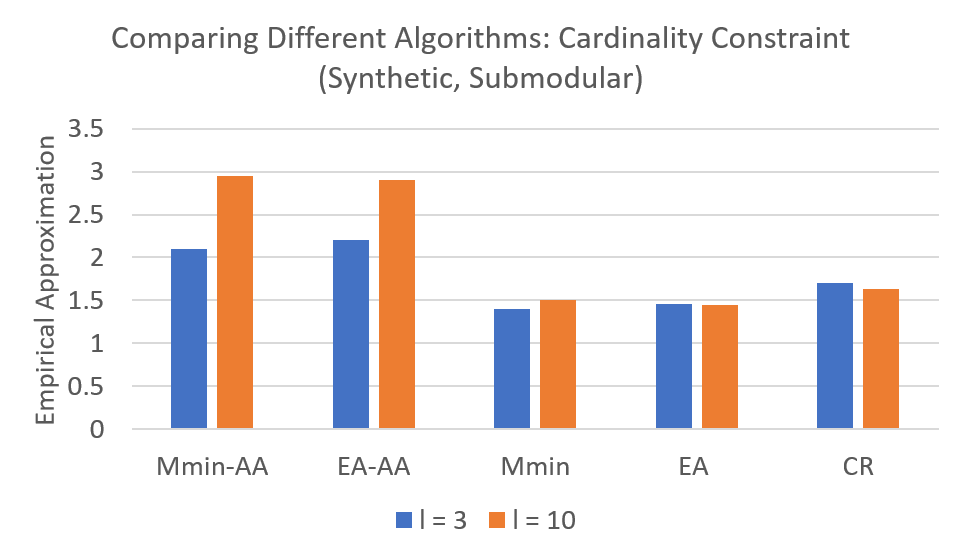}
        \includegraphics[width = 0.35\textwidth]{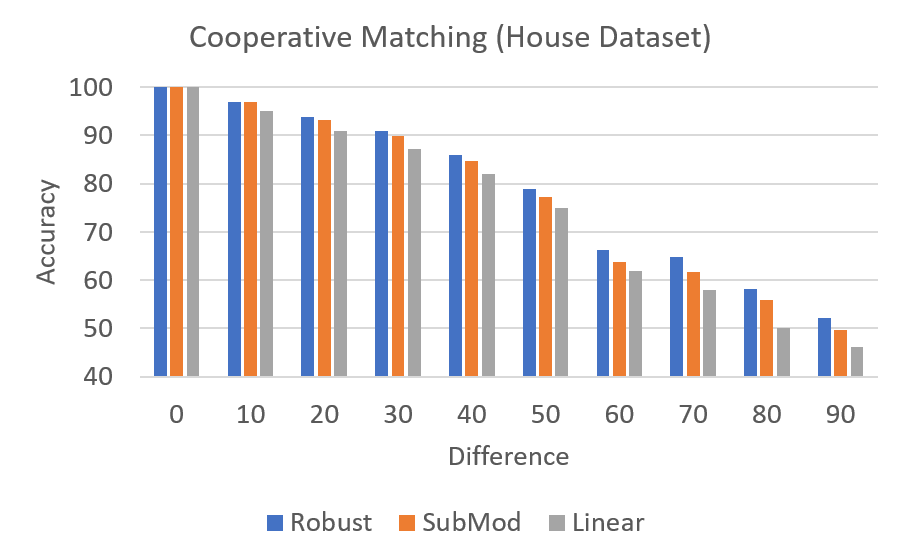}
        \includegraphics[width = 0.35\textwidth]{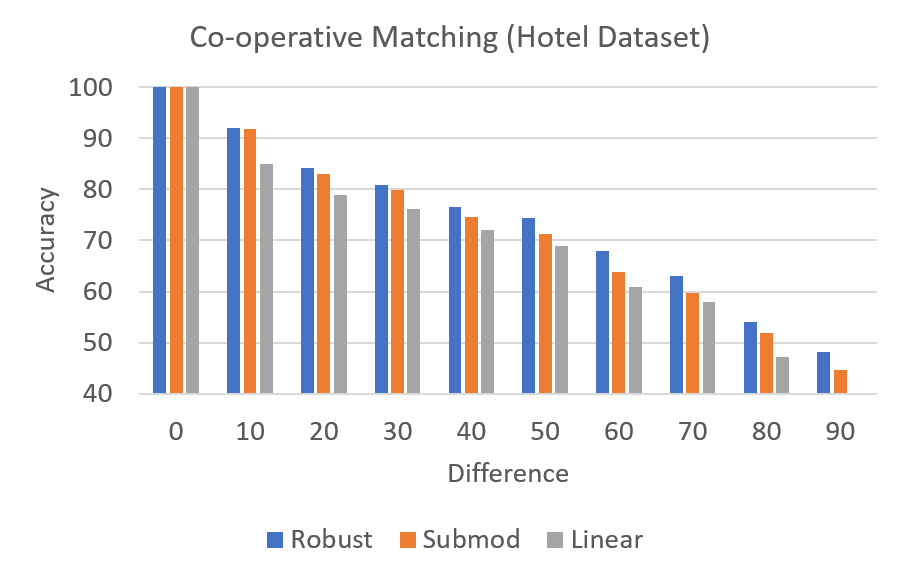}
    \caption{Top Row: Synthetic experiments with left: matching and modular, center: matching and submodular and right: cardinality and submodular. Bottom Row:  Co-operative Matching with left: House and right: Hotel datasets.}
    \label{fig:my_label}
\end{figure*}

\subsection{Continuous Relaxation Algorithm}
Here, we use the continuous relaxation of a submodular function. In particular, we use the relaxation $\max_i \hat{f_i}(x), x \in [0, 1]^{|\mathcal V|}$ as the continuous relaxation of the original function $\max_i f_i(X)$ (here $\hat{f}$ is the \lovasz{} extension). Its easy to see that this is a continuous relaxation. Since the \lovasz{} extension is convex, the function $\max_i \hat{f_i}(x)$ is also a convex function. This means that we can exactly optimize the continuous relaxation over a convex polytope. The remaining question is about the rounding and the resulting approximation guarantee due to the rounding. Given a constraint $\mathcal C$, define the up-monotone polytope similar to~\cite{iyer2014monotone} $\hat{\mathcal P_{\mathcal C}} = \mathcal P_{\mathcal C} + [0,1]^{|\mathcal V|}$. We then use the observation from~\cite{iyer2014monotone} that all the constraints considered in Table 1 can be expressed as:
\begin{align}
\hat{\mathcal P_{\mathcal C}} = \{x \in [0,1]^n | \sum_{i \in W} x_i \geq b_W \mbox{, for all $W \in \mathcal W$}\}. 
\end{align}

We then round the solution using the following rounding scheme. Given a continuous vector $\hat{x}$ (which is the optimizer of $\max_i \hat{f_i}(x), x \in \hat{\mathcal P_{\mathcal C}}$, order the elements based on $\sigma_{\hat{x}}$. Denote $X_i = [\sigma_{\hat{x}}[1], \cdots, \sigma_{\hat{x}}[i]]$ so we obtain a chain of sets $\emptyset \subseteq X_1 \subseteq X_2 \subseteq \cdots \subset X_n$. Our rounding scheme picks the smallest $k$ such that $X_k \in \hat{\mathcal C}$. Another way of checking this is if there exists a set $X \subset X_k$ such that $X \in \mathcal C$. Since $\hat{\mathcal C}$ is up-monotone, such a set must exist.  The following result shows the approximation guarantee.
\begin{theorem}
Given submodular functions $f_i$ and constraints $\mathcal C$ which can be expressed as $\{x \in [0,1]^n | \sum_{i \in W} x_i \geq b_W$ for all $W \in \mathcal W$\} for a family of sets $\mathcal W = \{W_1, \cdots\}$, the continuous relaxation scheme achieves an approximation guarantee of $\max_{W \in \mathcal W} |W| - b_W + 1$. If we assume the sets in $\mathcal W$ are disjoint, the integrality bounds matches the approximation bounds. 
\end{theorem}
\begin{proof}
The proof of this theorem is closely in line with Lemma 2 and Theorem 1 from~\cite{iyer2014monotone}. We first show the following result. Given monotone submodular functions $f_i, i \in 1, \cdots, l$, and an optimizer $\hat{x}$ of $\max_i \hat{f}(x)$, define $\hat{X_{\theta}} = \{i: \hat{x}_i \geq \theta\}$. We choose $\theta$ such that $\hat{X_{\theta}} \in \mathcal C$. Then  $\max_i f_i(\hat{X_{\theta}}) \leq 1/\theta \max_i f_i(X^*)$ where $X^*$ is the optimizer of $\min_{X \in \mathcal C} \max_i f_i(X)$. To prove this, observe that, by definition $\theta 1_{\hat{X_{\theta}}} \leq \hat{x}$\footnote{$1_A$ is the indicator vector of set $A$ such that $1_A[i] = 1$ iff $i\in A$.}. As a result, $\forall i, \hat{f_i}(\theta 1_{\hat{X_{\theta}}}) = \theta f_i(\hat{X_{\theta}}) \leq \hat{f_i}(\hat{x})$ (this follows because of the positive homogeneity of the \lovasz{} extension. This implies that $\theta \max_i f_i(\hat{X_{\theta}}) \leq \hat{f_i}(\hat{x}) \leq \min_{x \in \mathcal P_{\mathcal C}} \hat{f_i}(\hat{x}) \leq \min_{X \in \mathcal C} \max_i f_i(X)$. The last inequality holds from the fact that the discrete solution is greater than the continuous one since the continuous one is a relaxation. This proves this part of the theorem.

Next, we show that the approximation guarantee holds for the class of constraints defined as $\{x \in [0,1]^n | \sum_{i \in W} x_i \geq b_W\}$. 

From these constraints, note that for every $W \in \mathcal W$, at least $b_W \leq |W|$ elements need to ``covered''.  Consequently, to round a vector $x \in \hat{\mathcal P_{\mathcal C}}$, it is sufficient to choose $\theta = \min_{W \in \mathcal W} x_{[b_W, W]}$ as the rounding threshold, where $x_{[k, A]}$ denotes the $k^{\mbox{th}}$ largest entry of $x$ in a set $A$. 
The worst case scenario is that the $b_W - 1$ entries of $x$ with indices in the set $W$ are all $1$, and the remaining mass of $1$ is equally distributed over the remaining elements in $W$. In this case, the value of $x_{[b_W, W]}$ is $1/(|W| - b_W + 1)$. 
Since the constraint requires $\sum_{i \in W} x_i \geq b_W$, it must hold that $x_{[b_W, W]} \geq 1/(|W| - b_W + 1)$. Combining this with the first part of the result proves this theorem.
\end{proof}

We can then obtain the approximation guarantees for different constraints including cardinality, spanning trees, matroids, set covers, edge covers and vertex covers, matchings, cuts and paths by appropriately defining the polytopes $\mathcal P_{\mathcal C}$ and appropriately setting the values of $\mathcal W$ and $\max_{W \in \mathcal W} |W| - b_W + 1$. Combining this with the appropriate definitions (shown in Table 3), we get the approximation bounds in Table 1.

\subsection{Analysis of the Bounds for Various Constraints} 
Given the bounds in Tables 1 and 2, we discuss the tightness of these bounds viz-a-via the hardness. In the case when $l$ is a constant, MMin achieves tight bounds for Trees, Matchings and Paths while the EA achieves tight bounds up to $\log$ factors for knapsack constraints. In the case when $l$ is not a constant, MMin achieves a tight bound up to $\log$ factors for spanning tree constraints. The continuous relaxation scheme obtains tight bounds in the case of vertex covers. In the case when the functions $f_i$ have curvature $\kappa = 1$ CR also obtains tight bounds for edge-covers and matchings. We also point out that the bounds of average approximation (AA) depend on the average case curvature as opposed to the worst case curvature. However, in practice, the functions $f_i$ often belong to the same class of functions in which case all the functions $f_i$ have the same curvature.

\section{Experimental Results}
\subsection{Synthetic Experiments} 
The first set of experiments are synthetic experiments. We define $f_j(X) = \sum_{i = 1}^{|\mathcal C_j|} \sqrt{w(X \cap C_{ij})}$ for a clustering $\mathcal C_j = \{C_1, C_2, \cdots, C_{|\mathcal C_j|}$. We define $l$ different random clusterings (with $l = 3$ and $l = 10$). We choose the vector $w$ at random with $w \in [0, 1]^n$. We compare the different algorithms under cardinality constraints $(|X| \leq 10)$ and matching constraints. For cardinality constraints, we set $n = 50$. For matchings, we define a fully connected bi-partite graph with $m = 7$ nodes on each side and correspondingly $n = 49$. The results are over 20 runs of random choices in $w$ and $\mathcal C$'s and shown in Figure 1 (top row). The first plot compares the different algorithms with a modular function (the basic min-max combinatorial problem). The second and third plot on the top row compare the different algorithms with a submodular objective function defined above. In the submodular setting, we compare the average approximation baselines (MMin-AA, EA-AA), Majorization-Minimization (MMin), Ellipsoidal Approximation (EA) and Continuous Relaxation (CR). For the Modular cases, we compare the simple additive approximation of the worst case function, the Max-approximation and the quadratic approximation.  We use the graduated assignment algorithm~\cite{gold1996graduated} for the quadratic approximation to solve the quadratic assignment problem. For other constraints, we can use the efficient algorithms from~\cite{buchheim2018quadratic}

Figure 2 (left) shows the results in the modular setting.  first that as expected that int he modular setting, the simple additive approximation of the max function doesn't perform well and the quadratic approximation approach performs the best. Figure 2 center and right show the results with the submodular function under matching and cardinality constraint. Since the quadratic approximation performs the best, we use this in the MMin and EA algorithms. First we see that the average approximations (MMin-AA and EA-AA) don't perform well since it optimizes the average case instead of the worst case. Directly optimizing the worst case performs much better. Next, we observe that MMin performs comparably to EA though its a simpler algorithm (a fact which has been noticed in several other scenarios as well~\cite{rkiyersemiframework2013,nipssubcons2013,jegelka2011-inference-gen-graph-cuts}

\subsection{Co-operative Matchings} 
In this set of experiments, we compare \textsc{Robust-SubMin} in co-operative matchings. We follow the experimental setup in~\cite{iyer2019near}. We run this on the House and Hotel Datasets~\cite{caetano2009learning}. The results in Figure 2 (bottom row). The baselines are a simple modular (additive) baseline where the image correspondence problem becomes an assignment problem and the co-operative matching model from~\cite{iyer2019near} which uses a single clustering. For the robust model, we use the class of functions~\cite{iyer2019near}, except with multiple clusterings instead of one. Note that these clusterings are over the pixels in the two images, which then induce a clustering on the set of edges. In particular, we construct $l$ clusterings of the pixels: $\{(\mathcal E^1_1, \cdots, \mathcal E^1_k), \cdots, (\mathcal E^l_1, \cdots, \mathcal E^l_k)\}$ and define a robust objective: $f_{robust}(S) = \max_{i = 1:l} \sum_{j = 1}^k \psi_j(w(S \cap \mathcal E^i_j)), S \subseteq \mathcal E$. We run our experiments with $l = 10$ different clustering each obtained by different random initializations of k-means. From the results, we see that the robust technique consistently outperforms a single submodular function. In this experiment, we consider all pairs of images with the difference in image numbers from 1 to 90 (The $x$-axis in Figure 2 (bottom row) -- this is similar to the setting in~\cite{iyer2019near}).

\section{Conclusions}
In this paper, we study Robust Submodular Minimization from a theoretical and empirical perspective. We study approximation algorithms and hardness results. We propose a scalable family of algorithms (including the majorization-minimization algorithm) and theoretically and empirically contrast their performance. In future work, we would like to address the gap between the hardness and approximation bounds, and achieve tight curvature-based bounds in each case. We would also like to study other settings and formulations of robust optimization in future work.\looseness-1

\bibliography{ecai}
\end{document}